\documentclass[11pt]{article}
\usepackage[utf8]{inputenc}
\usepackage{amsmath, amssymb, amsthm, color, hyperref}
\usepackage{algorithm2e}
\usepackage{graphicx}
\newcommand{\omitthis}[1]{}
\newcommand{\E}{{\bf E}\,}

\newtheorem{definition}{Definition}

\newtheorem{example}{Example}

\newtheorem{proposition}{Proposition}
\newtheorem{remark}{Remark}

\newcommand{\comment}[1]{}

\newcommand{\ba}{\mbox{\boldmath $a$}}

\newcommand{\sC}{{\cal C}}
\newcommand{\bD}{\mbox{\boldmath $D$}}

\newcommand{\sG}{{\cal G}}

\newcommand{\sN}{{\cal  N}}

\newcommand{\sR}{{\cal R}}
\newcommand{\bbR}{\mathbb{R}}

\newcommand{\sT}{{\cal T}}

\newcommand{\bv}{\mbox{\boldmath $v$}}
\newcommand{\bV}{\mbox{\boldmath $V$}}

\def\mAth{\mathsurround=0pt}
\def\eqalign#1{\,\vcenter{\openup1\jot \mAth
  \ialign{\strut\hfil$\displaystyle{##}$&$\displaystyle{{}##}$\hfil
    \crcr#1\crcr}}\,}

\newcommand{\eq}{\begin{equation}}
\newcommand{\eeq}{\end{equation}}

\long\def\omitthis#1{}
\newcommand{\ist}{{\rm({\it i}\,{\rm )}}}
\newcommand{\ind}{{\rm({\it ii}\,{\rm )}}}
\newcommand{\ird}{{\rm({\it iii}\,{\rm )}}}
\newcommand{\ifourth}{{\rm({\it iv}\,{\rm )}}}

\newcommand{\itemi}{\item[\ist]}
\newcommand{\itemii}{\item[\ind]}
\newcommand{\itemiii}{\item[\ird]}
\newcommand{\itemiv}{\item[\ifourth]}

\newcommand{\given}{\,|\,}
\newcommand{\smalleq}{\hskip -1pt = \hskip -1pt}

\makeatletter
\DeclareRobustCommand{\nand}{\mathbin{\mathpalette\n@and@or\land}}
\DeclareRobustCommand{\nor}{\mathbin{\mathpalette\n@and@or\lor}}

\newcommand{\n@and@or}[2]{%
  \vphantom{#2}%
  \ooalign{$\m@th#1#2$\cr\hidewidth$\m@th#1\sim$\hidewidth\cr}%
}

\makeatother

\begin{document}
\title{\bf On the Use of Generative Models\\ in Observational Causal Analysis} 
\author{Nimrod Megiddo\thanks{IBM Almaden Research Center, San Jose, California}}
\date{June 2023}
\maketitle
\parindent=0pt

\begin{abstract}
The use of a hypothetical generative model was been suggested for causal analysis of observational data. 
The very assumption of a particular model is a commitment to a certain set of variables and therefore to a certain
set of possible causes. Estimating the joint probability distribution of can be useful for predicting values of variables
in view of the observed values of others, but it is not sufficient for inferring causal relationships. The model describes 
a single observable distribution and cannot a chain of effects of intervention that deviate from the observed distribution.
\end{abstract}

\section{\hskip -15pt. Introduction}
The possibility of inferring causal relations from observational data has been discussed since the 19th century.
It is well understood that observational data, by which we mean samples from a certain probability space, 
can be used for estimating the underlying probability distribution, 
and such samples can be used to infer dependencies among random variables.
However, inference of causality relations requires some ``intervention.''
A real intervention amounts to collecting data from a different distribution.
A ``mental'' intervention is a certain calculation that relies not only on the sampled data
but also on a certain hypothetical generative model.

A generative model may be a useful means for infusing into the analysis some external knowledge in addition to the observed data.
When a model is adopted, a certain set of random variables is assumed, and their joint probability distribution
can be estimated.
The prevalent model \cite{pearl2009causality} is that of a directed acyclic graph (DAG).
The DAG generative model was previously questioned \cite{dawid2010beware}.  
One objection to the DAG is that some assumptions of the represented by directed edges of the graph may be hard to justify.

Here we raise a more general objection, 
which is independent of the particular type of the generative model. 
We discuss the implication of assuming {\em any} specific generative model with regard to causal inference. 

Drug discovery is one of the areas in which causal analysis is hoped to have a breakthrough. 
This is one of the cases where the observational data is about a sample of human individuals.
We focus here on such cases for the sake of intuitive illustration. 
Also, for simplicity we restrict attention to binary variables.


\section{\hskip -15pt. Sample spaces, attributes and random variables}
Data analysis is typically performed under the assumption that the data  was sampled from some finite probability space whose
sample space is
$ \Omega = \{\omega_1,\ldots,\omega_n\}$.
For intuition, we assume the sample space is a finite population of humans, 
denoted $N=\{1,\ldots,n\}$.
Suppose we wish to find whether a certain ``treatment'' causes a certain ``response,'' 
or whether some treatment causes a stronger response than another treatment.
Some data about individuals of the population, 
including possible treatments and responses, 
is represented by {\em attributes}
$a_j: N \to \bbR$, $j\in M = \{1,\ldots,m\}$. 
For simplicity, we assume all the attributes are binary, i.e., $a_j(i) \in\{0,1\}$.
Thus, each individual $i \in N$ has an associated $m$-tuple of attribute values $\ba_i = (a_{i1},\ldots,a_{im})$, 
where $a_{ij} = a_j(i)$, $j\in M$.
In particular, $a_{ij}$ may indicate whether the individual $i$ received a certain treatment $\sT$.
Alternately, $a_{ij}$ may indicate whether the individual $i$ presented a certain response $\sR$.

The data can be analyzed in different contexts, which are formalized by different probability spaces as described below.
In the {\em purely-observational} setting, a random sample of individuals, $S\subset N$ is taken, 
and the $m$-tuples $\ba_i$, $i\in S$, are observed. 
In other settings, there may also exist attributes whose values are determined only after sampling the individuals.%
\footnote{For clarity, we reserve the notation $a_j$ only for attributes whose values are fixed in advance.}
Such attributes are actually mappings $b_j : S \to\{0,1\}$, i.e., their domain is the sample rather than the population.
For example, in a randomized controlled trial, individuals are first sampled, and a subgroup is thereafter selected for treatment; the members of
the sample are later observed for responses. 
The values of the attributes indicating treatment and response are determined {\em after} the sampling.
Human individuals of course have many attributes whose values are fixed before the sampling and assignment to treatment groups.

When one individual $i\in N$ is sampled, each attribute $a_j: N \to \{0,1\}$, $j\in M$, 
defines a random variable $X_j$ such that $X_j = a_j(i)$.   
Attributes whose values are determined after the individual has been sampled also define respective random variables if the values are generated probabilistically.

We distinguish two scenarios as follows. 

{\em Scenario} I. For each individual $i\in N$, it is already known whether $i$ was treated with $\sT$, 
which we denote by $t(i)=1$; otherwise, $t(i)=0$.  
Also, it is known whether $i$ presented the response $\sR$, 
which we denote by $r(i)=1$; otherwise, $r(i)=0$. 
Let $X:N\to \{0,1\}$ be the binary random variable defined by the attribute $t$, namely, if $\omega = i$, then
$X(\omega)= t(i)$, i.e., $X(\omega)=1$ if and only if $\omega$ was treated with $\sT$.
Similarly, let $Y:N\to \{0,1\}$ be the binary random variable define by the attribute $r(\cdot)$, namely, if $\omega = i$, then
$Y(\omega)= r(i)$, i.e., $Y(\omega)=1$ if and only if $\omega$ was treated with $\sR$.
This scenario represents to the so-called observational analysis. 
It is only observational because the analyst cannot control who gets the treatment.

{\em Scenario} II.   The decision of whether a sampled  $i \in N$  gets the treatment is made outside of the original space.
It can be considered an ``intervention."
Suppose the underlying ``random trial" of picking $i$ from $N$ is expanded, and only  {\em after} $i$ has been picked, 
the decision of whether  $i$ receives the treatment $\sT$ is made independently by a certain random process. 
For example, it can be a symmetric process in the sense that for every every pair of 
subsets $T_1, T_2 \subseteq S$ of the sample, if $|T_1|=|T_2|$ then $T_1$ and $T_2$ have the same probability of being the set of those members who get treated. 
Finally, after treatment of the members of $T\subset S$, each member of $S$ is checked for the response $\sR$.
This scenario represents a randomized controlled trial.

\subsection{A remark on different probability spaces and graphical models}   \label{twovar}
Consider the simple case of two binary random variables $X$ and $Y$.
It can arise in different probability spaces. 
First, the sample space may be the minimal one:
\[ \Omega = \{(0,0), (0,1),(1,0),(1,1) \} ~.\]
The joint probability distribution of $X$ and $Y$ can be described by a so-called graphical model \cite{jordan1999learning}.
One possibility of a graphical model may state that the value of $X$ is picked to be equal to $1$ with probability $p$, 
and thereafter,  the value of $Y$ is picked to be equal to $1$ with probability $p_1$ if $X=1$, 
and with probability $p_0$ if $X=0$.  
Then, the probability measure on $\Omega$ is as follows.
\[ \eqalign{ \Pr((0,0)) =&\ (1-p)\,(1-p_0)\cr
                  \Pr((0,1)) =&\ (1-p)\,p_0\cr
                  \Pr((1,0)) =&\ p\,(1-p_1)\cr
                  \Pr((1,1)) =&\ p\,p_1 ~.}\]
Here the values of $X$ and $Y$ determine a {\em unique} point in $\Omega$.
Note, that $X$ and $Y$ are not the only binary random variables on this space. 
In many cases of interest, however, $\Omega$ represents a set of individuals $\omega$, each with
equal probability of $\frac{1}{n}$, and for every $(x,y)\in \{0,1\}^2$ there may exist more than one $\omega\in\Omega$
such that $(X(\omega),~Y(\omega) ) = (x,y)$.  
With multiple points $\omega\in\Omega$ such that $(X(\omega),~Y(\omega) ) = (x,y)$,
 it is not true that the pair $(X,Y)$ determines a unique point of the sample space.
 Obviously, a graphical model that involves only $X$ and $Y$ describes {\em only} how the values of these
 variables are generated. 
 Thus, by definition, such a model rules out any other random variable $Z$ that in reality may
 be the cause of both $X$ and $Y$.  
 It is obvious and well-known that ignoring $Z$ (which may even be unknown) may lead to wrong causal conclusions.

\section{\hskip -15pt. Purely-observational data} \label{PO}
When the data is purely observational, the sample space is simply the population $N$.
For simplicity, we consider sampling with replacement.
Each $i\in N$ has the same probability $\frac{1}{n}$ to be sampled during each of the $s$ sampling steps. 
Denote this probability space by $\sN^s$.
The vectors of attribute values $\ba_1,\ldots,\ba_n$ are fixed in advance. 
In this setting, when one individual is sampled, the $m$ attributes define random variables $X_1,\ldots,X_m$, so that
for $j=1,\ldots,m$, if individual $i$ is sampled, then $X_j = a_{ij}$.
If $s$ individuals are sampled, we have random variables $X_{kj}$, ($k=1,\ldots,s$, $j=1,\ldots,m$), so that
if the $k$th member of the sample is the individual $i\in N$, then $X_{kj} = a_{ij}$.
This can be formalized using the space $\sN^s$,
where each sample point is an $s$-tuple of elements of $N$.
A sample $S = \{i_1,\ldots,i_s\}$ comes with the data $\bD_S=(\ba_{i_1}, \ldots, \ba_{i_s})$.
Note that even without the vectors of attribute values $\ba_1,\ldots,\ba_n$, there implicitly exist  $2^n-1$ binary random variables $Y_T$
($\emptyset \neq T\subseteq N$), defined by
$Y_T = 1$ if and only if $i \in T$, where $i$ is the sampled individual.

Consider the question of whether a certain treatment $\sT$ causes a certain response $\sR$.
Suppose the set of attributes $\{a_1,\ldots,a_m\}$ contains \ist\ an attribute $a_j$ so that for every $i \in N$, 
$a_j(i) = 1$ if and only if individual $i$ was treated with $\sT$, 
and \ind\ another attribute $a_\ell$ so that for every $i\in N$, 
$a_\ell(i) = 1$ if and only if $i$ presented the response $\sR$.

\begin{example} \label{ex:2}\rm
Suppose a random sample $S$ of size $200$ was taken from $N$ and
the counts within the set of the first 100 individuals are given in the following table:

\begin{center}
\begin{tabular}{c|c|c|}
   & $a_\ell=0$ & $a_\ell=1$  \\
   \hline
 $a_j=0$ & 81 & 9 \\
 \hline
 $a_j=1$ & 9   & 1 \\
 \hline
 \end{tabular}
 \end{center}
and the counts within the set of the last 100 are given in the following table:
\begin{center}
\begin{tabular}{c|c|c|}
   & $a_\ell=0$ & $a_\ell=1$ \\
   \hline
 $a_j=0$ & 1 &  9 \\
 \hline
 $a_j=1$ & 9   & 81 \\
 \hline
 \end{tabular}
\end{center}
Thus, 90\% of the members of the first set were not treated, and  90\% of the members of the first set did not present the response.
Also,  90\% of the members of the first set were treated, and  90\% of the members of the first set  presented the response.
Denote by $X_t$ and $X_r$ the random variables associated with the attributes $a_j$ and $a_\ell$, respectively.
Thus,  if $\omega = i$, then $X_t(\omega) = a_j(i)$ and $X_r(\omega) = a_\ell(i)$.
It seems that within each set, $X_t$ and $X_r$ are independent. 
The combined counts are given in the following table:

\begin{center}
\begin{tabular}{c|c|c|}
   & $a_\ell=0$ & $a_\ell=1$ \\
   \hline
 $a_j=0$ & 82 &  18 \\
 \hline
 $a_j=1$ & 18   & 82 \\
 \hline
 \end{tabular}
 \end{center}
\end{example}
which suggests that the treatment and the response are correlated.
We cannot conclude that the treatment causes the response because within each of the two sets
the treatment and the response seem independent. 
In the purely-observational case we do not know how individuals were picked for treatment. 
However, it is unlikely that the decision to treat was independent of the set to which the individual belonged.

The choice to look at the first $100$ and the last $100$ separately is arbitrary. 
In fact, there are $2^n-2$ different ways to partition the sample,
and one of them even matches perfectly the set of individuals who presented the response. 
However, that partition is useless for inference because it could be identified only after the data of treatments and responses has been generated. 
This observation motivates the use of a randomized controlled trial treatment, as we show later.

\paragraph{A remark on unknown confounding variables.}
The above discussion leads to an observation as follows.
It is often said that observational data cannot be used to infer a causal relation, 
because the treatment and response may be correlated by some unknown confounding variable, 
and it is not known whether such a confounding variable exists.
It is interesting to note though that if the response appears in the observational data to be dependent on the treatment, 
then, mathematically speaking, the treatment binary variable {\em itself} confounds the treatment and response variables.
There are of course several more such ``confounders,'' namely those whose sets overlap the treatment set significantly.
Therefore, the existence of a confounding variable is not questionable at all.
The fact that the data set at hand is a random sample from the population is not related to the lack of information with regard to how
it was decided which individuals to treat.
Of course, it is interesting to find whether a more ``natural'' confounding variable exist, but this is not a 
mathematically well-defined question.

\section{\hskip -15pt. Randomized controlled trial of treatment}  \label{RCT}
%
The concept of a randomized controlled trial of size $s$ with $t$ treated subjects can be formalized as follows.
\begin{definition} \label{def:trial} \rm
A sequence $S=(i_1,\ldots,i_s)$ is generated by sampling from $N$, without replacement, $s$ times.  
The $s$ members of $S$ are the ones picked for participation in the trial. 
Then, a subsequence $T\subset S$ of size $t$ is picked at random for receiving the treatment $\sT$. 
The members of $U = S\setminus T$ either receive a ``placebo'' treatment or are not treated at all.
Later, each member of $S$ is examined to see whether it presents the response $\sR$.
\end{definition}
Let $\sC_0$ denote the probability space whose sample space is the set of all pairs of sequences $(S,T)$
as in Definition \ref{def:trial}.  
Thus, the size of $\sC_0$ is equal to
\[ n(n-1)\cdots (n-s+1) \cdot {s \choose t} = \frac{n!\,s!}{(n-s)!\, t!\, (s-t)!} ~.\]
Let $\sC$ denote the probability space whose sample space is the set of triples $(S,T,R)$,
where $(S,T)$ is a point in $\sC_0$ and $R\subseteq S$.
Define binary random variables $T_1,\ldots,T_s$  and $R_1,\ldots,R_s$
so that for $k=1,\ldots,s$, $T_k = 1$ if and only if $i_k \in T$,
and $R_k = 1$ if and only if $i_k \in R$.

The scenario of a randomized controlled trial is very different from the one described in Section \ref{PO}. 
Fix $a_j:N\to\{0,1\}$ ($j\in M$)  to be any attribute whose value is determined prior to the sampling.
Denote
\[ |a_j| = |\{i\in N: a_j(i) = 1\}| ~. \]
The attribute $a_j$ can also be interpreted as
the subset $N_j\subseteq N$, where $i \in N_j$ if and only if $a_j(i) = 1$.  
The random choice of the sample $S = \{i_1,\ldots,i_s\}$ induces binary random variables $A_{kj}$ ($k=1,\ldots,s$, $j\in M$)
where $A_{kj} = a_j(i_k)$. 
\begin{proposition} \label{prop:indep}
For $k=1,\ldots,s$ and $j\in M$, the variables $T_k$ and  $A_{kj}$ are independent.
\end{proposition}
\begin{proof}
By construction, for $k=1,\ldots,s$,   \phantom\qedhere
\[ \Pr(T_k=1) ~=~ \Pr(T_k=1\given A_{kj}=1) ~=~ \Pr(T_k=1\given A_{kj}=0) ~=~ \frac{t}{s}~.   ~~~~~~~~~~~~~\hfil\qed\]
\end{proof}
Note that the variables $R_k$ and $A_{kj}$ are not necessarily independent.
In particular, an individual $i \in N_j$ may be more likely to present the response $\sR$ (regardless of treatment),
so it may happen that 
\[ \Pr(R_k=1\given A_{kj}=1) > \Pr(R_k=1\given A_{kj}=0)~.\]

For each individual $ i\in N$,
denote by $\tau_i$ the probability that $i$ would present the response $\sR$ if $i$ were treated with $\sT$, and 
denote by $\nu_i$ the probability that $i$ would present the response $\sR$ if $i$ were {\em not} treated with $\sT$.
Denote 
$$\overline\tau = \frac{1}{n}\cdot\sum_{i=1}^n\tau_i$$ and $$\overline\nu = \frac{1}{n}\cdot\sum_{i=1}^n\nu_i~.$$

\paragraph*{Effect of treatment.} For $k=1,\ldots,s$, the response in a treated individual satisfies:
\begin{equation}  \label{eq:treated}  
\Pr(R_k=1\given T_k=1) 
= 
 \sum_{i=1}^n \Pr(i_k = i\given T_k=1) \cdot \Pr(R_k= 1\given  (i_k = i)\cap ( T_k=1)) 
=  \frac{1}{n} \cdot  \sum_{i=1}^n \tau_i  
 = \overline \tau    ~,
\end{equation}
and the response in an untreated individual satisfies
\begin{equation} \label{eq:untreated}  
 \Pr(R_k=1\given T_k=0) 
= \sum_{i=1}^n \Pr(i_j = i\given T_j=0) \cdot \Pr(R_j= 1\given  (i_j = i)\cap ( T_j=0)) 
 =  \frac{1}{n} \cdot \sum_{i=1}^n  \nu_i 
  =\overline \nu   ~.
\end{equation} 
Hence the expected effect of treatment:
\eq \Pr(R_k=1\given T_k=1)  - \Pr(R_k=1\given T_k=0) =  \overline \tau  -  \overline \nu ~.
\eeq

\paragraph{The (unconditional) probability of response.} It follows that
\[  \eqalign{
 \Pr(R_k=1) 
 = &\   \sum_{i=1}^n \Pr(i_k = i) \cdot \Pr(R_k= 1\given  i_k = i)  \cr
= &\  \frac{t}{s} \cdot \overline \tau +  \left(1 - \frac{t}{s}\right) \cdot \overline \nu   ~,\cr
}
\]
\paragraph*{Attributes in samples.}
First,
\[  
\eqalign{
\Pr(A_{kj}= &\ 1) = \frac{| \{i : a_j(i) = 1 \} | }{n} = \frac{|a_j|}{n} \cr
\mbox{\rm and} ~~~~~~~~~~~~~~~~~~~~~~ & \cr
\Pr( (T_k= &\ 1) \cap (A_{kj}=1) ) = \frac{t}{s}  \cdot \frac{|  a_j |}{n}  ~.
}\]
Since
\[ 
 \Pr((R_k=1)\cap(T_k=1) \cap (A_{kj}=1)) = 
 \sum_{i:\ a_j(i)=1} \tau_i  \cdot \frac{t}{s} \cdot \Pr(i_k=i)   
  = 
    \frac{t}{s}\cdot \frac{1}{n}  \cdot \sum_{i:\ a_j(i)=1} \tau_i ~,
    \]
it follows that 
\[ \Pr(R_k =1\given (T_k=1) \cap (A_{kj}=1) )
 = \frac{\Pr((R_k=1)\cap(T_k=1)\cap (A_{kj}=1))}{\Pr( (T_k=1) \cap (A_{kj}=  1))  }
 =  \frac{1}{|a_j|}\sum_{i:\ a_j(i)=1} \tau_i ~.\]
 Similarly,
\[ \Pr(R_k=1\given (T_k=0) \cap (A_{kj}=1) )
 =  \frac{1}{|a_j|} \sum_{i:\, a_j(i)=1} \nu_i ~,\]
\[ \Pr(R_k=1\given (T_k=1) \cap (A_{kj}=0) )
 =  \frac{1}{n-|a_j|}\sum_{i:\, a_j(i)=0} \tau_i \]
 and 
 \[ \Pr(R_k=1\given (T_k=0) \cap (A_{kj}=0) )
 =  \frac{1}{n-|a_j|}\sum_{i:\, a_j(i)=0} \nu_i ~.\]
 
 \paragraph*{Effect of attribute.}
\[   \eqalign{\Pr(R_k=1 \given A_{kj} =1) 
   =&\ \Pr(T_k=1 \given A_{kj}=1)\cdot  \Pr(R_k =1\given (T_k=1) \cap (A_{kj}=1) )\cr
   &+ \Pr(T_k=0 \given A_{kj}=1)\cdot  \Pr(R_k =1\given (T_k=0) \cap (A_{kj}=1) ) \cr
  =&\  \frac{t}{s} \cdot   \frac{1}{|a_j|}\sum_{i:\, a_j(i)=1} \tau_i
    + \left(1 - \frac{t}{s}\right) \cdot \frac{1}{|a_j|} \sum_{i:\, a_j(i)=1} \nu_i
}\]
\[   \eqalign{\Pr(R_k=1 \given A_{kj} =0) 
   =&\ \Pr(T_k=1 \given A_{kj}=0)\cdot  \Pr(R_k =1\given (T_k=1) \cap (A_{kj}=0) )\cr
   &+ \Pr(T_k=0 \given A_k=0)\cdot  \Pr(R_k =1\given (T_k=0) \cap (A_{kj}=0) ) \cr
  =&\   \frac{t}{s} \cdot   \frac{\sum_{i:\, a_j(i)=0} \tau_i}{n -|a_j|}
    + \left(1 - \frac{t}{s}\right) \cdot \frac{1}{n-|a_j|}\sum_{i:\, a_j(i)=0} \nu_i 
}\]
\[ \eqalign{\Pr &(R_k = 1  \given A_{kj} =1) -  \Pr(R_k=1 \given A_{kj} =0)  \cr
= &\   \frac{t}{s} \cdot \left(   
        \frac{1}{|a_j|}\sum_{i:\,  a_j(i)=1} \tau_i 
     -  \frac{1}{n -|a_j|} \sum_{i:\,  a_j(i)=0} \tau_i\right) 
  +  \left(1 - \frac{t}{s}\right) \cdot\left(\frac{1}{|a_j|}\sum_{i:\,  a_j(i)=1} \nu_i  
  - \frac{1}{n-|a_j|}\sum_{i:\,  a_j(i)=0} \nu_i \right) \cr
  }\]
 
\paragraph*{Conditional effect of treatment.} Given a value of an attribute:
 \[  \eqalign{ \Pr(R_k =1 &  \given (T_k=1) \cap (A_{kj}=1) ) -  \Pr(R_k=1\given (T_k=0) \cap (A_{kj}=1) )   \cr
&\ =  \frac{1}{|a_j|}\sum_{i:\, a_j(i)=1} \tau_i -  \frac{1}{|a_j|}\sum_{i:\, a_j(i)=1} \nu_i \cr
&\ =  \frac{1}{|a_j|} \sum_{i:\, a_j(i)=1} (\tau_i - \nu_i)
~.}\]
\[ \eqalign{\Pr(R_k=1 & \given (T_k=1) \cap (A_{kj}=0) ) - \Pr(R_k=1\given (T_k=0) \cap (A_{kj}=0) ) \cr
&\ = \frac{1}{n-|a_j|}\sum_{i:\, a_j(i)=0} \tau_i  - \frac{1}{n-|a_j|}\sum_{i:\, a_j(i)=0} \nu_i  \cr
&\ =  \frac{1} {n-|a_j|} \sum_{i:\, a_j(i)=0}( \tau_i - \nu_i) ~.
} \]

 {\em Example.}  Suppose an attribute $a_j$ is very strongly correlated with the response $\sR$, so that
 if $a_j(i)=1$, then $\tau_i = \nu_i = 1$, and 
 if $a_j(i)=0$, then $\tau_i = \nu_i = 0$.
 In the purely-observational case, such an attribute may lead to a false conclusion about the effect of treatment on the response, if many of the $i$s with $a_j(i)=1$ are treated, and only a few of them are not treated.
In the controlled-randomized-trial setting, there is no such effect because $\tau_i = \nu_i$ for every
$i$, and therefore
 \[   \Pr(R_k=1  \given T_k =1) -  \Pr(R_k=1 \given T_k =0) =  \overline \tau - \overline\nu =0~.\]

\paragraph{The trial data set.}
Recall (see (\ref{eq:treated}) and (\ref{eq:untreated})) that for $k=1,\ldots,s$,
\[ \Pr(R_k=1  \given T_k =1) = \overline \tau  \]
and 
\[ \Pr(R_k=1  \given T_k =0) = \overline \nu ~.\]
The latter two can obviously be estimated from the trial results as follows.
Let $Z_k = 1$ if both $T_k=1$ and $R_k=1$;  otherwise, let $Z_k=0$  
In other words, $Z_k = T_k\cdot R_k$.
Thus, $ \overline \tau$ can be estimated by 
 \[  \rho_1  \equiv   \frac{1}{s} \cdot \sum_{k=1}^s \frac{ Z_k  }{t/s}
 =   \frac{1}{t} \cdot \sum^s_{k=1} T_k\cdot R_k~. \]
 Similarly, $ \overline \nu$ can be estimated by
 \[  \rho_0 \equiv \frac{1}{s-t} \cdot \sum^s_{k=1} (1-T_k) \cdot R_k ~.\]

 \begin{remark}\rm
In Example \ref{ex:2}, the number of treated individuals is $100$, of which $10$ are from the first $100$
and $90$ are from the last $100$. 
Such a difference is extremely unlikely if the selection of individuals for treatment is randomized as described above.

Let $A\subset N$ be a fixed subset. 
Suppose $S\subset N$ is sampled from a uniform distribution over all the subsets of $N$ of size $s$, and
$T\subset S$ is sampled from a uniform distribution over all the subsets of $S$ of size $t$.
Thus, $T$ is in fact sampled from a uniform distribution over all the subsets of $N$ of size $t$.
Thus, for $0\le r \le t$,
\[ \Pr(|T\cap A| \smalleq r)
    ~=~ \frac{  {|A| \choose r}\cdot {n - |A| \choose t - r }}{{n\choose t}} ~. \]
It is well known that the expectations are
\[ \E[|T\cap A| ] = t\cdot \frac{|A|}{n} \]
and
\[ \E[|T\setminus A| ] = t \cdot \frac{n-|A|}{n} \]
and the variance is
\[ \mathop{\rm var}( |T\cap A|) = \mathop{\rm var}( |T\setminus A|) 
= t \cdot  \frac{|A|}{n} \cdot \frac{n-|A|}{n} \cdot \frac{n - t}{n-1}
< t \cdot  \frac{|A|}{n} \cdot \frac{n-|A|}{n} ~.\]
The Chernoff bounds for the binomial distribution imply
\[ \Pr( |T\cap A| > t\cdot \frac{|A|}{n}\cdot (1 + \varepsilon)
< \exp \left(- \varepsilon^2\cdot\frac{n^2}{2\, t |A|\cdot(|n-|A|)}
  \right)
  \]
and
\[ \Pr( |T\cap A| < t\cdot \left(\frac{|A|}{n} + \varepsilon \right)
< \exp \left(- \varepsilon^2\cdot\frac{n^2}{2\,t|A|\cdot(|n-|A|)}
  \right) ~.
  \]

Denote 
\[  \xi  = \frac{|T\cap A|}{t}  \]
and 
\[ \eta = \frac{|T\setminus A|}{n -|A|} ~.\]

Note that a fixed set $A$ as above corresponds to an attribute $a_A: N\to \{0,1\}$, 
where $a_A(i) = 1$ if and only if $i\in A$.  
Thus, the random sample $S$ and the randomized choice of the subset $T\subset S$ of treated individuals guarantee that for every fixed attribute $a_j(\cdot)$, 
the attribute value  $a_j(i_k)$ of a sampled individual $i_k$, $k=1,\ldots,s$,  and its membership in $T$ are
independent random variables.
This implies that, with high probability, the relative frequency of $i_k$s with $a_j(i_k) =1$ in $T$ and in $S\setminus T$ are very close.  
In particular, suppose $A$ is the set $n/2$ individuals $i$ with the largest effect of treatment $\overline \tau_i -
\overline \nu_i$.
If the set of treated people contains a higher proportion from $A$ than the set of untreated people, then it may appear that treatment is more effective on the average than the true effect.
The randomized controlled trial aims to avoid such mistake.
\end{remark}

\section{\hskip -15pt. Observational data together with a generative model}

\subsection{A generative distribution model}
The causal analysis of observational data relies on the assumption that the observed data was generated according to a certain generative model, 
which is supposed to be {\em a valid representation of reality.}
The analysis is based exclusively on the combination of the observed data and the model.
Thus, only variables that appear in the model can be considered to be causes for any model variables, i.e.,
implicit in the model are {\em all} the possible causes.
Causal analysis aims to quantify the relative strength of various possible causes.
Note that if the data set represents a sample of individual humans from a certain population $N$, 
then, 
despite the fact that each subset of $N$ defines a binary attribute with a corresponding random variable, 
only the variables that are explicitly included in the model are considered to be possible causes. 
In this respect, {\em the model rules out the possibility of any unknown confounding variable}. 

\begin{remark}\rm
The question of how to derive a generative model is of course very important, 
but here we focus on the implication of adopting any model for the purpose of causal analysis.
A generative model for the generation of observational data amounts to a {\em single} joint probability distribution of the
observable variables. 
For causal analysis, at least {\em two} joint distributions are required, namely, 
the distribution that arises when a certain variable $v_j$, the potential cause, is {\em forced} to be a constant equal to $1$ and
the distribution that arise when $v_j$ is {\em forced} to be a constant equal to $0$.
The generative model does not address such forcing at all.  
It only describes how the data is generated without any forcing.
\end{remark}

At least in the context of sampling human individuals from a population $N$, it cannot be assumed that these individuals are identical.  
In this section, we continue to assume the existence of attributes $a_j: N\to \{0,1\}$, $j \in M$, 
even though the values of these attributes are never observed, so they do not appear in the model.
If an attribute is observable, then we should represent it within the model and eliminate it from the set of attributes $\{a_1,\ldots,a_m\}$.
Thus, the values of all the observable attributes of an individual are assumed to be generated according 
to the model {\em after} the sampling from $N$ has taken place.

The variables that appear in the model are associated with members of the sample $S$.
Suppose the (binary) attributes appearing in the model are 
$v_j: S \to \{0,1\}$, $j=1,\ldots,q$, and the model specifies a joint probability distribution $P$ over their values.
Note that the individual $i_k$ still has well-defined attribute values $a_j(i_k)$, $j\in M$.
Thus, $P$ is a probability distribution over the space $\{0,1\}^q$.
A sample of $s$ individuals according to $P$ looks the same as a sample in the purely-observational case, that is,
a set $S=\{i_1,\ldots,i_s\}\subset N$ together with the respective sequence of tuples of attribute values 
$\bD_S=(\bv_{i_1}, \ldots, \bv_{i_s})$, 
where $\bv_{i_k} = (v_1 (i_k), \ldots, v_q (i_k))$, $k=1,\ldots,s$.

The distribution $P$ is not sufficient for causal analysis. 
Causal analysis also requires more detail about the order or partial order in which the values of variables were generated.
The partial order can be described by a directed acylic graph (DAG). 
The choice of a specific DAG has its own challenges, but this is not the topic of this discussion.
Our point here is that the very assumption committing to a set of observed values and their distribution raises difficulties from the viewpoint of causal analysis.

Note that the vectors $\bv_{i_k}$, $k=1, \ldots, s$, of the respective sampled individuals are sampled unconditionally according to $P$, i.e., 
regardless of the values $a_j(i)$ ($j=1,\ldots,m$, $i=1,\ldots,n$), 
which are fixed in advance outside of the generative model.
Thus, the data is observational, 
but there is also a specific hypothetical model that describes how the observable data was generated.  

Suppose an individual $i\in N$ is picked at random.
For each $A \subset N$, let $X_A$ be random variable such that $X_A = 1$ if $i\in A$, and $X_A=0$ otherwise.
Denote  by $\bV$ the random vector of attributes
$(v_1(i),\ldots,v_q(i))$ that are sampled (independently) from the distribution $P$. 
We thus have
\begin{proposition}\label{prop1}
For every $A\subseteq N$ and every vector $\bv \in \{0,1\}^q$,  
\[ \Pr((\{X_A=1\})\cap (\bV = \bv)) = \frac{|A|}{|N|}\cdot P(\bv)~, \]
i.e., the vector $\bV$ is independent of the value of the attribute defined by the set $A$.
\end{proposition}
\begin{remark}\rm
It appears that we must distinguish between attributes whose values are fixed in advance for each individual and
attributes whose values are sampled when an individual is selected for the sample. 
The former gives rise to random variables when individuals are picked randomly for the sample. 
The latter constitute random variables by definition.
Proposition \ref{prop1} states the independence of the former random variables and the attributes that appear in the model. 
The assumption that the model is a valid representation of reality implies that the variables that
do not appear in the model must also be independent of the ones that appear in the model.
\end{remark}

\begin{remark}\rm
There is no objection here to estimating the joint probability distribution of the observable variables and using it for predictions.
The objection  here is only to using this distribution for causal inference.
\end{remark}

\subsection{On DAG models}
One possibility of a generative model is formalized as an acyclic directed graph (DAG) $G=(V,E)$, 
whose $q$ vertices correspond to the specific attributes $v_j: S \to \{0,1\}$, $j=1,\ldots,q$,
together with probability distributions as follows. 
First, there are the respective probability distributions
of the attribute values whose vertices in $G$ have no parents, which can be called the {\em initial vertices}.
These attribute values are assumed to be stochastically {\em  independent.}
Second, if a vertex is not initial, then it has parents, and for every combination of attribute values of the parents, 
a probability distribution of the attribute value of this vertex is specified,
given the attribute values of its parent vertices.  
This generative model gives rise to a probability space $\sG$ whose sample space is $\{0,1\}^q$,
where the probability of each $q$-tuple can be calculated from the probability distributions associated with the various vertices of $G$.  


 \begin{remark}\rm
 In medical studies, $X$ may represent some behavior or treatment, and $Y$ may represent a later condition.
 A model that involves only $X$ and $Y$ cannot constitute ``a valid representation of reality'' because there are many
 attributes of an individual subject whose values are determined before the value of $X$ is.
 \end{remark}
 
 \begin{remark}\rm
 Consider the question of whether an {\em uncontrollable} variable $X$ can cause a variable $Y$.
Obviously, the question cannot be answered by running a controlled experiment.
Suppose $X$ predicts $Y$ very well. 
Thus, observing $X$ could be very valuable if some action must be taken which whose result depends on the value of $Y$. 
However, what would be the practical value of knowing whether $X$ causes $Y$?
We cannot influence the value of $X$, so it does not really matter, and preparing for $Y$ in view of the value of $X$ would 
be the same regardless of whether $X$ causes $Y$. 
In contrast, if the variable $X$ is controllable, then it is very valuable to know
whether $X$ causes $Y$ because then it may be possible to control $Y$ rather than just prepare for it.
 \end{remark}

 \begin{remark}\rm
 The generative-model-based causal analysis proceeds by fixing the value of a causal variable in the two possible ways (i.e., ``true'' or ``false'') and 
 evaluating the respective resulting probabilities of the effect variable.
 It has pointed been out in  \cite{geneletti2005aspects}, including a specific example, 
 that such an ``intervention'' in the model does not necessarily represent what might happen in reality.
 This is not surprising because the model is only supposed to represent how the data was generated,
 and the intervention is definitely not included in the data generation without the intervention.
 \end{remark}

 \begin{center}
\bibliographystyle{plain}
\bibliography{causality_probspaces}
\end{center}
\section*{Appendix - an example}
 Imagine a model with only two variables $X$ and $Y$.
 Suppose the value of $X$ is observed first, and the value of $Y$ is observed second.
 Suppose the data suggests that 
 $$p\equiv \Pr(X=1) = 0.5~,$$ 
 $$p_1\equiv \Pr(Y=1\,|\,X=1) = 0.905$$
  and also
 $$p_0\equiv \Pr(Y=0\,|\,X=0) = 0.905~.$$   
 Thus, the value of $X$ predicts the value of $Y$ very well.
 Since the model involves only $X$ and $Y$, the conclusion is that $X$ causes $Y$ quite strongly.
 Now, in reality there may exist a variable $Z$ that is not observed at all but has the following properties:
 \begin{enumerate}
  \itemi
 The value of $Z$ is determined before $X$, and $\Pr(Z=1) = 0.5$,
 \itemii
 $\Pr(X=1\,|\,Z=1) = \Pr(X=0\,|\,Z=0) = 0.95$
 \itemiii
  $\Pr(Y=1\,|\,Z=1) = \Pr(Y=0\,|\,Z=0) = 0.95$
  \itemiv
  $X$ and $Y$ are conditionally independent given $Z$.
 \end{enumerate}
 Note that these properties are consistent with the two-variable model since
 \[ \Pr(X=1) = \Pr(Z=1) \,\Pr(X=1\,|\,Z=1) + \Pr(Z=0) \,\Pr(X=1\,|\,Z=0) = 0.5\cdot 0.95 + 0.5\cdot 0.5 = 0.5 \]
 (and, similarly, $\Pr(Y=1)=0.5$), and also
 \[ \eqalign{ \Pr(Y=1\,|\,X=1) 
 =&\         \Pr(Z=1\,|\,X=1) \cdot \Pr(Y=1\,|\,Z=1)
            +  \Pr(Z=0\,|\,X=1) \cdot \Pr(Y=1\,|\,Z=0) \cr
 =&\         \Pr(X=1\,|\,Z=1)   \cdot\frac{\Pr(Z=1)}{\Pr(X=1) }   \cdot \Pr(Y=1\,|\,Z=1) \cr
            &+  \Pr(X=1\,|\,Z=0)   \cdot\frac{\Pr(Z=0)}{\Pr(X=1) }  \cdot \Pr(Y=1\,|\,Z=0) \cr
    =&\ 0.95^2 + 0.05^2 = 0.905 
 }\]
 and
  \[ \eqalign{ \Pr(Y=0\,|\,X=0) 
 =&\         \Pr(Z=1\,|\,X=0) \cdot \Pr(Y=0\,|\,Z=1)
            +  \Pr(Z=0\,|\,X=0) \cdot \Pr(Y=0\,|\,Z=0) \cr
 =&\         \Pr(X=0\,|\,Z=1)   \cdot\frac{\Pr(Z=1)}{\Pr(X=0) }   \cdot \Pr(Y=0\,|\,Z=1) \cr
            &+  \Pr(X=0\,|\,Z=0)   \cdot\frac{\Pr(Z=0)}{\Pr(X=0) }  \cdot \Pr(Y=0\,|\,Z=0) \cr
    =&\ 0.05^2 + 0.95^2 = 0.905 ~.
 }\]
 Thus, if $Z$ were known, it would be deemed the cause of each of $X$ and $Y$. 
 \end{document}
-----------
If $A$ and $B$ are independent, then
\[  \Pr(\neg A \cap \neg B) = 1 - \Pr(A \cup B) = 1 - \Pr(A) - \Pr(B) + \Pr(A\cap B) 
=  1 - \Pr(A) - \Pr(B) + \Pr(A)\cdot\Pr(B) = (1-\Pr(A))\cdot(1-\Pr(B) = \Pr(\neg A) \cdot \Pr(\neg B) \]
so also $\neg A$ and $\neg B$ are independent.
 \end{document}
 \newpage
 Conjugate to a function $f(x)$ is the function 
 \[ f^*(y) = \sup_x \{y^\top x - f(x) \} \]
 Example1
 \[  f(x) = ax + b \]
 \[  f^*(y) = \sup_x \{ y\,x - ax - b \} = 
 \begin{cases} 
   -b     & ~~ \mbox{if $y =a$ } \\
   \infty &~~ \mbox{if $y \not=a$}
   \end{cases}
 \]  
 Example 2.
 \[ f(x) = x^2 \]
  \[  f^*(y) = \sup_x \{ y\,x -  x^2 \} = y^2 / 4 \]
  \[ f^{**}(x) = \sup_y\{  xy - y^2 / 4\} =  x^2 = f(x) \]

\end{document}